\documentclass[letterpaper, 10 pt, conference]{ieeeconf}  

\IEEEoverridecommandlockouts                              

\usepackage{graphicx}
\usepackage{amsmath, amsfonts, amssymb} 
\usepackage{algorithmic, algorithm}
\usepackage{acro}
\usepackage{bm}
\usepackage{siunitx}
\usepackage{subcaption}
\usepackage{booktabs}
\usepackage{multirow}
\usepackage{placeins}
\usepackage{cite}
\usepackage{hyperref}
\usepackage{caption}
\usepackage{tabularx} 

\usepackage[english]{babel}
\newtheorem{theorem}{Theorem}

\newif\ifshowtable

\newif\ifshowfigure
\showfiguretrue

\title{\LARGE \bf Real-time Iteration Scheme for Diffusion Policy}

\author{Yufei Duan$^{1}$, Hang Yin$^{2}$ and Danica Kragic$^{1}$
\thanks{This work was supported by Knut and Alice Wallenberg Foundation and Swedish Research Council (2019-00748). This work was also partially supported by the Wallenberg AI, Autonomous Systems and Software Program (WASP) funded by the Knut and Alice Wallenberg Foundation. The computations were enabled by the supercomputing resource Berzelius provided by National Supercomputer Centre at Linköping University and the Knut and Alice Wallenberg foundation. Hang Yin acknowledges the support from Innovation Foundation Denmark under the Grand Solutions project Fast and Efficient Robotic Automation via Reuse of Data (FERA).}
\thanks{$^{1}$Division of Robotics, Perception and Learning, KTH Royal Institude of
Technology, Stockholm,
        {\tt\small \{yufeidu, dani\}}@kth.se}%
\thanks{$^{2}$Department of Computer Science, University of Copenhagen, Copenhagen,
        {\tt\small \{hayi\}}@di.ku.dk}%
}

\begin{document}

\maketitle
\thispagestyle{empty}
\pagestyle{empty}

\begin{abstract}

Diffusion Policies have demonstrated impressive performance in robotic manipulation tasks. However, their long inference time, resulting from an extensive iterative denoising process, and the need to execute an action chunk before the next prediction to maintain consistent actions limit their applicability to latency-critical tasks or simple tasks with a short cycle time. While recent methods explored distillation or alternative policy structures to accelerate inference, these often demand additional training, which can be resource-intensive for large robotic models. In this paper, we introduce a novel approach inspired by the Real-Time Iteration (RTI) Scheme, a method from optimal control that accelerates optimization by leveraging solutions from previous time steps as initial guesses for subsequent iterations. We explore the application of this scheme in diffusion inference and propose a scaling-based method to effectively handle discrete actions, such as grasping, in robotic manipulation. The proposed scheme significantly reduces runtime computational costs without the need for distillation or policy redesign. This enables a seamless integration into many pre-trained diffusion-based models, in particular, to resource-demanding large models. We also provide theoretical conditions for the contractivity which could be useful for estimating the initial denoising step. Quantitative results from extensive simulation experiments show a substantial reduction in inference time, with comparable overall performance compared with Diffusion Policy using full-step denoising. Our project page with additional resources is available at: {\href{https://rti-dp.github.io/}{https://rti-dp.github.io/}}


\end{abstract}

\section{INTRODUCTION}

Robotic manipulation has seen significant advancements enabled by diffusion models. These models~\cite{poole2022dreamfusiontextto3dusing2d, reuss2023goalconditionedimitationlearningusing, prasad2024consistencypolicyacceleratedvisuomotor,ze20243ddiffusionpolicygeneralizable}, especially Diffusion Policy (DP)~\cite{chi2024diffusionpolicyvisuomotorpolicy}, have demonstrated success in a wide range of tasks, improving control, adaptability, and generalization in complex manipulation scenarios.

However, one main drawback of diffusion models is the slow inference process. Standard diffusion models begin inference from a standard Gaussian distribution and refine it through a denoising process with hundreds of steps. This impedes high-frequency control in demanding tasks, such as contact-rich or high-speed manipulations, which require continuous and timely correction, to prevent irrecoverable errors or even outright failures. 

The inability to react quickly limits the application of diffusion-based policies in real-world robotics and bottleneck task throughput in repetitive execution. The prolonged waiting time for policies to generate actions results in sluggish, non-smooth robot movements, reducing productivity and making them significantly less responsive compared to human performance. Consequently, addressing the inference efficiency in diffusion-based policies is critical to broadening their applicability in robotic practice.
\begin{figure} [t]
        \centering
        \includegraphics[width=\linewidth]{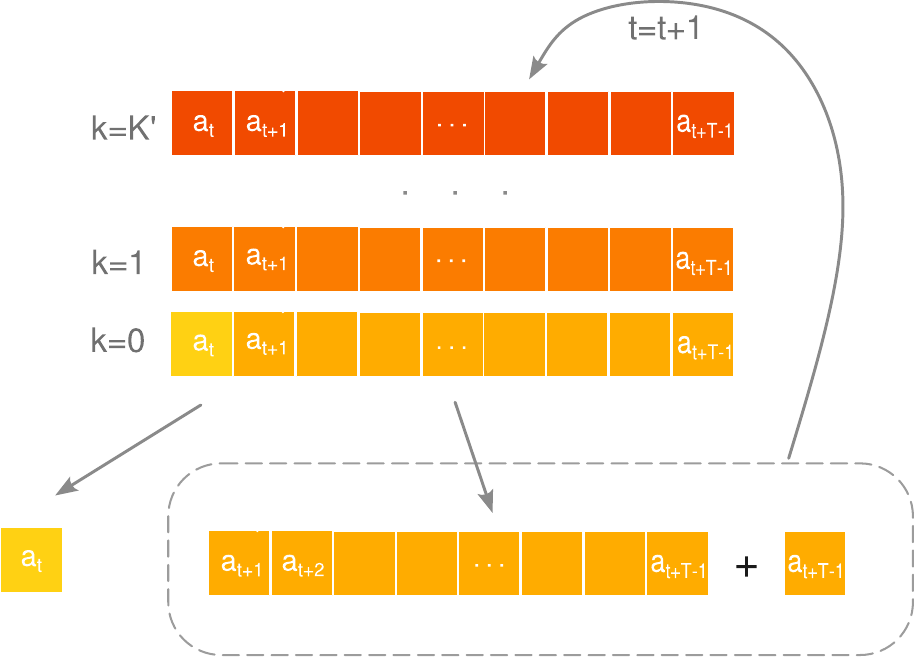}

    \caption{Illustration of Real-Time Iteration Scheme for Diffusion Policy (RTI-DP). At each time step, the predicted action sequence serves as the initial guess for the subsequent inference with truncated denoising steps. This sequence is updated by removing the executed action and appending the newly predicted action at the end, facilitating efficient and continuous policy adaptation.}
    \label{fig:enter-label}
\end{figure}

Recent works have explored various approaches to accelerating inference in diffusion models. One prominent direction involves distillation techniques~\cite{prasad2024consistencypolicyacceleratedvisuomotor}, which aim to condense the iterative denoising process into a one-step consistency model~\cite{song2023consistencymodels}. While this method achieves faster inference, they often come with drawbacks such as high computational costs during training and reduced policy quality and diversity~\cite{meng2023distillationguideddiffusionmodels}. Another line of research focuses on designing alternative policy structures~\cite{heg2024streamingdiffusionpolicyfast, chen2025responsivenoiserelayingdiffusionpolicy}, where the policy structure itself is redesigned to facilitate faster decision-making while maintaining performance. Although these methods can be effective, they are often impractical to apply to pre-trained large-scale robotic models, as they require complete retraining. Given that retraining large models is computationally expensive and resource-intensive, there is a pressing need for solutions that enable efficient policy acceleration without requiring extensive model modifications or retraining. 


In this paper, we propose a real-time method, Real-Time Iteration for Diffusion Policy (RTI-DP) that requires neither retraining nor distillation. Our approach is inspired by the Real-Time Iteration Scheme~\cite{rti} in optimal control, which leverages predictions from previous steps as initialization for the next step, significantly accelerating optimization. Specifically, our method first computes an initial action chunk using the same approach as the standard diffusion policy. For subsequent action chunks, it utilizes the previous one as an initial guess and applies denoising on top of it with a reduced number of time steps. A key insight of our approach is the importance of initialization, which stems from the spatiotemporal consistency of physical systems. This consistency implies that actions exhibit continuity and bounded changes over time, ensuring that initialization remains effective across consecutive action steps. As a result, RTI serves as a natural extension for diffusion policy inference. Theoretical results can be derived from the consistency assumption, stating with a well-chosen initial guess can efficiently converge to a high-quality prediction demanding hundreds of denoising steps. This analysis aids in identifying the optimal initial denoising step to maintain high inference quality.

Our contribution can be summarized as follows:
\begin{itemize}
    \item We propose a real-time inference method that leverages the initial guess from the previous prediction, enabling direct application to pre-trained models.
    \item We provide theoretical conditions for contractive errors by infererring from the proposed initial guess.
    \item We show the proposed method can empirically boost the real-time performance in various manipulation tasks through simulation experiments.
\end{itemize}


\section{RELATED WORK}

Diffusion models are capable of representing intricate multimodal distributions while maintaining robust training stability and resistance to hyperparameter fluctuations. They have been driven widespread use in robotics, spanning areas such as motion 
planning~\cite{janner2022planningdiffusionflexiblebehavior}, imitation learning~\cite{chi2024diffusionpolicyvisuomotorpolicy, prasad2024consistencypolicyacceleratedvisuomotor,reuss2023goalconditionedimitationlearningusing}, and grasp prediction~\cite{weng2024dexdiffusergeneratingdexterousgrasps}. Many current methods depend on generating trajectories step-by-step using full denoising from white noise, a characteristic feature of Denoising Diffusion Probabilistic Models (DDPMs)~\cite{ho2020denoisingdiffusionprobabilisticmodels}. However, implementing these approaches in real-world applications incurs significant computational overhead, posing challenges for practical deployment.

To address inference inefficiencies, recent work on consistency-based models has demonstrated significant speed improvements for generative models. Consistency Policy (CP)~\cite{prasad2024consistencypolicyacceleratedvisuomotor}, built upon Consistency Trajectory Models~\cite{kim2024consistencytrajectorymodelslearning}, was originally designed to accelerate image generation and has since been adapted to robotic policy learning. These models achieve performance comparable to Diffusion Policy while being significantly faster. However, a key drawback of consistency-based methods is their difficulty in preserving multimodal action distributions and maintaining stable training~\cite{prasad2024consistencypolicyacceleratedvisuomotor}.

Another approach to reducing inference time in diffusion models involves minimizing the number of denoising steps. Unlike DDPM’s stochastic solver, DDIM~\cite{song2022denoisingdiffusionimplicitmodels} reformulates the process as a deterministic ODE, allowing models to be trained with a large number of denoising steps but evaluated with fewer during inference. Similarly, EDM~\cite{karras2022elucidatingdesignspacediffusionbased} refines this strategy with modifications to preconditioning and weighting. However, despite their efficiency, adaptive step-size techniques like DDIM and EDM often compromise sample quality when the number of denoising steps is reduced~\cite{karras2022elucidatingdesignspacediffusionbased}.

Another line of research involves modifying the policy structure. Streaming Diffusion Policy (SDP)~\cite{heg2024streamingdiffusionpolicyfast} proposes a method generating a partially denoised action chunk while ensuring that the immediate next action is fully denoised. Similarly, ~\cite{chen2025responsivenoiserelayingdiffusionpolicy} introduces action chunks with varying noise levels, enabling future actions to be progressively denoised, thereby reducing the need for iteratively denoising at one prediction. Also,~\cite{reuss2023goalconditionedimitationlearningusing} provides a method to achieve 3-step denoising with proper choice of training and sampling algorithm.

On the other hand, a substantial body of research in optimal control has centered on methods that accelerate iterative solvers by providing high-quality initial guesses for each new timestep—commonly referred to as warm-starting. For instance, in Model Predictive Control (MPC), solutions from preceding timesteps can be reused to reduce computational overhead~\cite{rao1998application}. A notable contribution in this area is the Real-Time Iteration (RTI) scheme~\cite{rti}, which processes an approximate solution offline and refines it online using updated system information.


\section{REAL-TIME ITERATION SCHEME FOR DIFFUSION POLICY}

\subsection{Preliminaries}
The Denoising Diffusion Probabilistic Model~\cite{ho2020denoisingdiffusionprobabilisticmodels} employs a forward noising process that gradually adds Gaussian noise, generating a sequence of states that, in the diffusion-policy setting, represent the action $\mathbf{A}$:
\begin{equation}
    \mathbf{A}_k = \sqrt{\bar{\alpha}_k} \mathbf{A}_0 + \sqrt{1 - \bar{\alpha}_k} \boldsymbol{\epsilon},
\end{equation}
where $\bar{\alpha}_k = \Pi^t_{s=1} \alpha_s$ is the noise schedule and $\epsilon \sim \mathcal{N}(0, \mathbf{I})$. To reverse this process, the model learns a conditional denoising procedure that inverts this sequence:
\begin{equation}
\begin{aligned}
    \mathbf{A}_{k-1} = \frac{1}{\sqrt{\alpha_k}}(\mathbf{A}_{k} - \frac{\beta_k}{\sqrt{1-\bar{\alpha}_k}}\epsilon_\theta(\mathbf{A}_{k},k, \mathbf{O})) + \sigma_k\mathbf{Y}, \\
    ~\mathbf{Y}\sim \mathcal{N}(0, \mathbf{I}),
\end{aligned}
\end{equation}
where $\beta_t=1-\alpha_t$; $\mathbf{O}$ denotes the observation on which the diffusion process is conditioned; and $\epsilon_\theta(\mathbf{A}_{k}, \mathbf{O},k)$ is the learned function.

Conventionally, the reverse process starts from pure Gaussian noise, but this approach requires many denoising steps to recover the underlying data, leading to high computational costs and increased inference latency. 

By contrast, in the control domain, RTI~\cite{rti} is an efficient strategy for Nonlinear Model Predictive Control (NMPC) in time-critical applications. Traditional NMPC requires solving a nonlinear optimization problem at each sampling instant, which can be computationally prohibitive for large systems with fast dynamics, similar to diffusion models. The RTI scheme addresses this challenge through a single-iteration optimization strategy that leverages the similarity between consecutive control problems to achieve real-time performance.

Specifically, at each time step $t_k$, RTI formulates the NMPC problem as a nonlinear program (NLP):
\begin{equation}
    \begin{aligned}
        \min_{\mathbf{z}_{k:N}, \, \mathbf{u}_{k:N-1}} & \quad \sum_{i=k}^{N-1} \ell(\mathbf{z_i}, \mathbf{u_i}) + \ell_f(\mathbf{z_N}) \\
        \text{s.t.} & \quad \mathbf{z_{i+1}} = f(\mathbf{z_i}, \mathbf{u_i}), \quad i=k,\dots,N-1 \\
                    & \quad \mathbf{z_k} = \mathbf{x}(t_k)
    \end{aligned}
\end{equation}
where $\mathbf{\{z_i\}}$ is the state trajectory and $\mathbf{\{u_i\}}$ is the control sequence. One of the keys to RTI's computational efficiency lies in its initialization procedure. Rather than starting from arbitrary initial guesses, RTI employs:

\begin{enumerate}
    \item \textbf{Preparation:} RTI first computes an initial guess across the entire horizon and prepares the offline Jacobian matrix and the gradient vector.
    \item \textbf{Feedback response:} At each time step, RTI computes the incremental step changes and applies the resulting control to the real system.
    \item \textbf{Transition:} RTI updates the next initial guess by incorporating the step changes and shrinking the prediction horizon by one.
\end{enumerate}
 
By leveraging this scheme, computation time can be reduced by a factor of hundreds to thousands. 

Diffusion Policy and NMPC share fundamental similarities when interacting with physical controllable systems. Control methods operating in the real world must inherently respect the spatio-temporal consistency, as robots cannot change states discontinuously across space or time. 

\textbf{Spatiotemporal Consistency} in robot visuomotor policies is the property that a robot's trajectory evolves smoothly over time and in space. It ensures that state transitions are bounded in magnitude, such that the difference between the state at one time step and the next will not be arbitrarily large. 

Formally, let \( x_t^{\phi} \in \Phi \) denote the trajectory’s configuration expressed in a smooth manifold $\Phi$. The consistency constraint is expressed as:
\begin{equation}
  d( x^{\phi}_{t+1}, x^{\phi}_t) \;\le\; L,
  \quad \forall t \in \{0,1,\dots\},
\end{equation}
where $d(\cdot,\cdot)$ is a distance metric defined on the manifold $\Phi$, and \( L \) is a Lipschitz-like constant that bounds the maximum allowable change in the robot’s state between timesteps. This stipulates that the trajectory evolves continuously and smoothly over time in space, without abrupt changes between consecutive timesteps.

By harnessing the \emph{spatiotemporal} consistency properties, both Diffusion Policy and NMPC can leverage action guesses that effectively build on previous predictions. However, key differences remain: diffusion models evolve according to stochastic differential equations (SDEs), whereas NLP-based NMPC solves ordinary differential equations (ODEs). As a result, it requires distinct approaches for handling NMPC and Diffusion Policy. 

In the following section, we introduce Real-Time Scheme for Diffusion Policy, providing an analysis about its effectiveness on retaining performance with fewer denoising steps.

\subsection{Method}
Our method begins by computing the first action prediction using standard Diffusion Policy with full denoising. As this initial step is performed offline, computation time is not a primary concern. For subsequent predictions, we treat the outcome of the previous prediction as an initial guess and apply a reduced number of denoising steps to refine it. This process is repeated iteratively, with each new prediction building upon the last, enabling the policy to operate efficiently in real time—one step at a time—while maintaining fast inference.

\begin{algorithm}
\begin{algorithmic}
\caption{Real-time Iteration for Diffusion Policy}

\REQUIRE $M_\theta$: denoising model, $O_t$: observation at time $t$, $A_t = [a_t, a_{t+1}, \dots, a_{t+T-1}]$: action chunk at time $t$ with an action horizon of $T$, $G \sim \mathcal{N}(0, I)$: initial guess sampled from a standard Gaussian distribution, $K$: denoising steps
\STATE $t \gets 0$

\WHILE{NOT FINISHED}
\IF{$t=0$}
\STATE \COMMENT{Compute the initial action chunk with full-step denoising}
\FOR{$k \in FullDenoisingSteps$}
    \STATE $G \gets M_\theta(O_t, G,k)$
\ENDFOR
\STATE $A_t \gets G$
\ELSE
\STATE \COMMENT{Iterative refinement using previous predictions}
    \STATE $G \gets [a_{t+1}, a_{t+2}, \dots, a_{t+T-1}, a_{t+T-1}]$  \COMMENT{Shift and repeat last action}

    \FOR{$k \in K$}
        \STATE $G \gets M_\theta(O_t, G, k)$
    \ENDFOR
    \STATE $A_t \gets G$
\ENDIF
    \STATE Execute $a_t$
    \STATE $t \gets t + 1$
\ENDWHILE

\end{algorithmic}
\end{algorithm}

\subsubsection{Discrete action space}
Our method is based on the fundamental assumption of spatiotemporal consistency in action states, meaning that consecutive timesteps exhibit bounded changes in actions. However, this assumption may be violated, particularly in scenarios involving discrete actions, such as a binary command of grasping. In such cases, abrupt transitions can pose challenges for diffusion-based approaches, potentially leading to corrupted predictions.

To address this issue, we introduce a scaling factor to transform the discrete dimensions of data, allowing the change of discrete actions to be better captured within a continuous representation. This transformation enables the diffusion model to smoothly interpret and denoise actions, preserving trajectory consistency while retaining the expressiveness of the original discrete signals.

We explore two strategies:
\begin{itemize}
    \item \textbf{For pre-trained policies}, we directly scale the discrete \textit{initial guess} by a factor that remains manageable within the denoising steps. For example, for an initial guess of a discrete action $g_d$, $g_d = \frac{g_d}{10}$.
    \item \textbf{For retrainable tasks}, a more effective approach is to preprocess the discrete actions in the \textit{dataset} using a scaling factor, ensuring robust performance in handling discrete actions. For example, the discrete action $a_d$ in the dataset is changed to $\frac{a_d}{10}$.
\end{itemize}

\subsection{Why is initialization important?}
Diffusion models reveal the data distribution by employing a reverse diffusion process that starts from a predefined prior distribution, referred to here as the initialization. Typically, the prior is chosen as a unimodal Gaussian distribution due to its mathematical simplicity and analytical tractability. Since a Gaussian prior does not encode information about specific data modes, it provides a mode-agnostic initialization. Then the learned reverse diffusion process can effectively guide noisy samples toward diverse modes within the data distribution, facilitating flexible adaptation of the robot to various potential actions during inference. While this multimodality is beneficial for generating diverse behaviors, it is not always favorable during real-time execution, where consistency and stability are critical. As the policy execution progresses, the mode naturally collapses, as a system cannot follow multiple modes simultaneously. With a limited conditional context, the policy may oscillate between different modes, potentially becoming trapped in conflicting movement decisions instead of progressing toward the goal. This phenomenon is also observed in DP experiments, where performance deteriorates as the action horizon decreases from its optimal value~\cite{chi2024diffusionpolicyvisuomotorpolicy}. Therefore, during the robot execution, starting inference from a point near the same mode can alleviate this issue, reduce the chance of mode switch and keep consistency.

Another challenge arising from an uninformative initialization is the difficulty in achieving a fully closed-loop system while maintaining consistency. This challenge is exacerbated by the absence of mode consistency between consecutive predictions. To address this, in DP, an action chunk must be predicted to ensure a sequence of consistent actions. This design choice inherently results in only partially closed-loop behavior. This limitation stems from their open-loop operation between predictions, where a sequence of actions is generated from a single observation and executed without intermediate feedback. This design represents a trade-off necessitated by the long prediction time and the need to preserve mode consistency. Consequently, closing the loop becomes challenging, despite its fundamental importance for real-time control. 

Our method preserves multimodal capabilities by selecting the mode in the initial step based on the standard diffusion policy while ensuring stable execution through predictions from previous time steps, thereby maintaining consistency with local behavior modes. To achieve a fully closed-loop system, we adopt a one-action-per-prediction strategy, enabling continuous feedback without compromising performance. This approach maintains real-time adaptability while ensuring computational efficiency.

\subsection{Local Contractivity of Real-Time Iterations}
Assume we have a forward chain of actions $\{{\mathbf{A}_k}\}$ generated from the clean ${\mathbf{A}_0}$. Suppose we do not start from $\{{\mathbf{A}}_K\}$ in the reverse process, but from an intermediate step $K'<K$ with some $\mathbf{{A}}_{K'}$. Following the spatiotemporal consistency inherent in physical systems, the actions predicted from the observation $\mathbf{O}_t$ tend to closely ensemble those predicted at a previous time step with $\mathbf{O}_{t-1}$ if they belong to the same behavioral mode. Consequently, the estimated $\mathbf{\tilde{A}}_{K'}$ remains close to the true $\mathbf{A}_{K'}$, enabling the reverse chain to stay near the original trajectory and ultimately converge toward $\mathbf{A}_0$. 

Since denoising follows a Markov chain structure, each step depends only on the preceding one. For a contractive Markov, small errors are corrected rather than accumulated, ensuring stability.

\begin{theorem}[Local Contractivity in RTI-DP]
Consider a denoising diffusion probabilistic model (DDPM) with a total of $K$ steps and a noise schedule for the variance $\beta_k$. Let $\mathbf{A}_0$ be a clean data sample, and let $\mathbf{A}_k$ be the corresponding noisy sample at step $k$ obtained via the forward diffusion process and $\mathbf{O}$ is the observation on which the diffusion process is conditioned. Assume the learned  denoising function $\epsilon_\theta(\mathbf{A},k, \mathbf{O})$ is L-Lipschitz in $\mathbf{A}$ for all $k$.

Suppose we initialize the reverse diffusion process at step $K'<K$ with an estimate $\tilde{\mathbf{A}}_{K'}$ such that $\|\tilde{\mathbf{A}}_{K'} - \mathbf{A}_{K'}\|$ is small. Then, after running the reverse process from step $K'$ down to $0$, the final recovered sample $\tilde{\mathbf{A}}_0$ satisfies:
\begin{equation}
    \|\tilde{\mathbf{A}}_0 - \mathbf{A}_0\| \leq C(K') \|\tilde{\mathbf{A}}_{K'} - \mathbf{A}_{K'}\|,
\end{equation}
where $C(K')=\prod_{k=1}^{K'} c_k$. $c_k$ is a constant dependent on the number of remaining steps $k$, the Lipschitz constant of the learned function and the noise schedule.

For well-behaved noise schedules (e.g., squared cosine), $c(k) < 1$ for all k. This ensures $C(K')<1$, making the reverse process contractive. Consequently, small initialization errors at step $K'$ decay exponentially, guaranteeing ${\mathbf{\tilde{A}}_0}$ stays close to $\mathbf{A}_0$.

\end{theorem}

\begin{proof}

Under standard DDPM assumptions, the reverse process of DDPM is a Markov chain. If we denote the true posterior by
\begin{equation}
\begin{aligned}
    q(& \mathbf{A}_{k-1}\mid \mathbf{A}_k, \mathbf{A}_0, \mathbf{O}) \\
     & \;\propto\; q(\mathbf{A}_k \mid \mathbf{A}_{k-1}, \mathbf{A}_0, \mathbf{O})\, q(\mathbf{A}_{k-1}\mid \mathbf{A}_0, \mathbf{O}),
\end{aligned}
\end{equation}
it is a Gaussian with mean
\begin{equation}\label{eq:posterior-mean-final}
   \mu_k(\mathbf{A}_k, \mathbf{A}_0, \mathbf{O}) \;=\;
   \frac{1}{\sqrt{\alpha_k}}\left(\mathbf{A}_{k} - \frac{\beta_k}{\sqrt{1-\bar{\alpha}_k}}\epsilon_\theta(\mathbf{A}_{k},k,\mathbf{O})\right)
\end{equation}
and the standard deviation of ${\sigma}_k = \frac{1 - \bar{\alpha}_{k-1}}{1 - \bar{\alpha}_k}\,\beta_k$.

Assuming the learned denoising function, $\epsilon_\theta(\mathbf{A},k, \mathbf{O})$ is L-Lipschitz in $\mathbf{A}$, that is
\begin{equation}
\lVert \epsilon_\theta(\mathbf{A}_k, k, \mathbf{O}) - \epsilon_\theta(\tilde{\mathbf{A}}_k, k, \mathbf{O})  \rVert \leq L \lVert \mathbf{A}_k - \tilde{\mathbf{A}}_k \rVert.
\end{equation}

This implies that $\mu_k(\mathbf{A}_k)$ is Lipschitz continuous. For well-designed schedulers, $c_k<1$.
\begin{equation}
    \|\mu_k(\tilde{\mathbf{A}}_k) - \mu_k(\mathbf{A}_k)\| \leq c_k \|\tilde{\mathbf{A}}_k - \mathbf{A}_k\|.
\end{equation}
where $
    c_k = \frac{{\sqrt{(1-\bar{\alpha}_k)\alpha_k}+\sqrt{\bar{\alpha}_{k}}\,\beta_kL}}{\sqrt{(1-\bar{\alpha}_k)\alpha_k}}\ $.
Thus, we can get
\begin{equation}
       \|\tilde{\mathbf{A}}_{k-1} - \mathbf{A}_{k-1}\|
   \;\le\; c_k \,\|\tilde{\mathbf{A}}_k - \mathbf{A}_k\|.
\end{equation}


Define the error at step $k$ as $\boldsymbol{\delta}_k = \| \mathbf{\tilde{A}}_k - \mathbf{A}_k\|$. From the previous step’s contraction property, we get
\begin{equation}
    \mathbb{E}[\boldsymbol{{\delta}}_{k-1}] \le c_k \mathbb{E}[\boldsymbol{\delta}_k].
\end{equation}
Then iterating from \( k=K' \) down to \( k=0 \), we obtain:
\begin{equation}
    \mathbb{E}[\boldsymbol{\delta}_0] \leq \left(\prod_{k=1}^{K'} c_k\right) \mathbb{E}[\boldsymbol{\delta}_K].
\end{equation}

Let $C(K') = \prod_{k=1}^{K'} c_k$. For $c_k<1$, $C(K')$ decays exponentially with $K'$, proving contractivity. 
\end{proof}

The contractivity property ensures
\begin{itemize}
    \item \textbf{Stability}:
    Small initialization errors $\boldsymbol{\delta}$ shrink to negligible errors in $\mathbf{A}_0$.
    \item \textbf{Consistency}:
    All reverse processes starting near $\mathbf{A}_{K'}$ map to the same $\mathbf{A}_0$.
\end{itemize}

\begin{figure*}
\vspace{0.8cm} %
        \centering
        \includegraphics[width=0.8\linewidth]{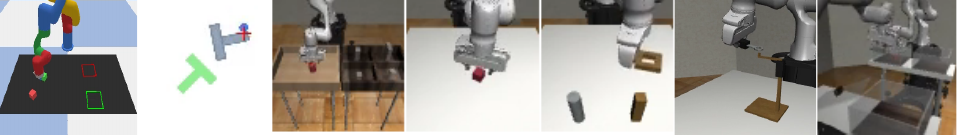}

    \caption{Illustrations of Simulation Experiments. Left to Right: Blockpush PushT, Can, Lift, Square, Tool Hang and Transport.}
    \label{fig:enter-label}
\end{figure*}

\begin{table*}[]
\resizebox{\textwidth}{!}{%
\begin{tabular}{@{}ccccccccccccc@{}}
\toprule
\multirow{2}{*}{Policy} & \multicolumn{2}{c}{Push-T}                                                                           & \multicolumn{2}{c}{Can}                                                             & \multicolumn{2}{c}{Lift}                                                            & \multicolumn{2}{c}{Square}                                                          & \multicolumn{2}{c}{ToolHang}                                                        & \multicolumn{2}{c}{Transport}                                                       \\ \cmidrule(l){2-13} 
                        & Score                               & \begin{tabular}[c]{@{}c@{}}Prediction\\ Time (ms)\end{tabular} & Score              & \begin{tabular}[c]{@{}c@{}}Prediction\\ Time (ms)\end{tabular} & Score              & \begin{tabular}[c]{@{}c@{}}Prediction\\ Time (ms)\end{tabular} & Score              & \begin{tabular}[c]{@{}c@{}}Prediction\\ Time (ms)\end{tabular} & Score              & \begin{tabular}[c]{@{}c@{}}Prediction\\ Time (ms)\end{tabular} & Score              & \begin{tabular}[c]{@{}c@{}}Prediction\\ Time (ms)\end{tabular} \\ \midrule
DP                      & 0.92/\textbf{0.91}                  & 816                                                            & \textbf{1.00}/0.98 & 833                                                            & \textbf{1.00/1.00} & 820                                                            & \textbf{0.98/0.97} & 828                                                            & \textbf{0.62}/0.55 & 829                                                            & 0.86/\textbf{0.83} & 823                                                            \\
SDP                     & 0.86/0.86                           & 438                                                            & \textbf{1.00/0.99} & 390                                                            & \textbf{1.00}/0.96 & 393                                                            & 0.94/0.87 & 465                                                            & \textbf{0.62}/0.57 & 391                                                            & \textbf{0.90}/0.81 & 392                                                            \\
RTI-DP-clip(ours)          & \multirow{2}{*}{\textbf{0.95/0.91}} & \multirow{2}{*}{\textbf{25}}                                   & \textbf{1.00/0.99} & 116                                                   & \textbf{1.00}/0.99 & 100                                                   & 0.92/0.89 & 123                                                   & 0.60/0.54          & 130                                                   & 0.80/0.69          & 145                                                   \\
RTI-DP-scale(ours)         &                                     &                                                                & \textbf{1.00}/0.93 & \textbf{100}                                                   & \textbf{1.00/1.00} & \textbf{36}                                                    & 0.92/0.90          & \textbf{78}                                                    & \textbf{0.62/0.58} & \textbf{25}                                                    & 0.86/0.77          & \textbf{56}                                                    \\ \bottomrule
\end{tabular}%
}
\caption{State-based Simulation Experiment Results (Max/Average Performance) of DP, SDP, RTI-DP.}
\label{table_state}
\end{table*}

\begin{table*}[]
\resizebox{\textwidth}{!}{%
\begin{tabular}{@{}ccccccccccccc@{}}
\toprule
\multirow{2}{*}{Policy} & \multicolumn{2}{c}{Push-T}                                                                  & \multicolumn{2}{c}{Can}                                                             & \multicolumn{2}{c}{Lift}                                                            & \multicolumn{2}{c}{Square}                                                          & \multicolumn{2}{c}{ToolHang}                                                        & \multicolumn{2}{c}{Transport}                                                       \\ \cmidrule(l){2-13} 
                        & Score                      & \begin{tabular}[c]{@{}c@{}}Prediction\\ Time (ms)\end{tabular} & Score              & \begin{tabular}[c]{@{}c@{}}Prediction\\ Time (ms)\end{tabular} & Score              & \begin{tabular}[c]{@{}c@{}}Prediction\\ Time (ms)\end{tabular} & Score              & \begin{tabular}[c]{@{}c@{}}Prediction\\ Time (ms)\end{tabular} & Score              & \begin{tabular}[c]{@{}c@{}}Prediction\\ Time (ms)\end{tabular} & Score              & \begin{tabular}[c]{@{}c@{}}Prediction\\ Time (ms)\end{tabular} \\ \midrule
DP                      & \textbf{0.92/0.89}         & 1008                                                           & \textbf{1.00/0.98} & 820                                                            & \textbf{1.00/1.00} & 796                                                            & \textbf{0.96}/0.93 & 814                                                            & \textbf{0.86/0.82} & 825                                                            & \textbf{0.92/0.90} & 833                                                            \\
SDP                     & 0.80/0.78                  & 520                                                            & 0.98/0.98          & 474                                                            & \textbf{1.00}/0.99 & 459                                                            & \textbf{0.96}/0.94 & 465                                                            & 0.02/0.01          & 473                                                            & 0.86/0.85          & 531                                                            \\
CP                      & 0.69/0.65                  & \textbf{18}                                                    & 0.94/0.93          & \textbf{16}                                                    & \textbf{1.00}/0.99 & \textbf{16}                                                    & 0.86/0.84 & \textbf{16}                                                    & 0.23/0.20          & \textbf{16}                                                    & 0.88/0.83          & \textbf{24}                                                    \\
RTI-DP-clip(ours)              & \multirow{2}{*}{0.88/0.86} & \multirow{2}{*}{36}                                            & 0.92/0.90          & 118                                                            & \textbf{1.00/1.00} & 131                                                            & 0.88/0.80          & 134                                                            & 0.66/0.62          & 153                                                            & 0.88/0.81          & 110                                                            \\
RTI-DP-scale(ours)              &                            &                                                                & \textbf{1.00}/0.93 & 100                                                            & \textbf{1.00/1.00} & 36                                                             & 0.86/0.85          & 182                                                            & 0.70/0.68          & 56                                                             & 0.86/0.77          & 56                                                             \\ \bottomrule
\end{tabular}%
}
\caption{Image-based Simulation Experiment Results (Max/Average Performance) of DP, SDP, CP, RTI-DP.}
\label{table_image}
\end{table*}


\begin{table}[]
\centering
\resizebox{0.8\columnwidth}{!}{%
\begin{tabular}{@{}cccc@{}}
\toprule
Policy  & p1        & p2        & \begin{tabular}[c]{@{}c@{}}Prediction \\ Time (ms)\end{tabular} \\ \midrule
DP-DDPM & 0.44/0.37 & 0.18/0.17 & 818                                                             \\
SDP     & 0.32/0.3  & 0.16/0.11 & 391                                                             \\
RTI-DP(ours)    & \textbf{0.54/0.45} & \textbf{0.24/0.19} & \textbf{25}                                                     \\ \bottomrule
\end{tabular}%
}
\caption{Simulation Result of Blockpush (Max/Average Performance) of DP, SDP, RTI-DP.}\label{table_block}

\end{table}

\subsubsection*{\textbf{Remark} On Estimating the Step \( K' \)}
In the current implementation, \( K' \) is chosen empirically. 
However, the contractivity property established above offers a theoretical foundation for choosing \( K' \) to automate the selection, striking a balance between denoising steps and errors. Possibly, one could estimate \( K' \) offline by running the full denoising steps to obtain \( \mathbf{A}_0 \), and selecting \( K' \) to minimize the deviation between the initial guess and the forward diffusion mean. 

\section{EXPERIMENTAL EVALUATION}
\label{sec:exp}
We demonstrate the effectiveness of Real-Time Iteration Scheme in achieving high accuracy and fast inference across diverse robotics benchmarks, spanning both image-based and state-based control in tasks of varying complexity and horizons. The specific questions we want to address with the experiments are:
\begin{itemize}
    \item To what extent does RTI-DP enhance inference time, and how significant is the speed-up compared to competing methods?
    \item How does RTI-DP compare to existing approaches in terms of performance across different robotics tasks? Will a faster but non-full-denoising inference degrade task performance? 
\end{itemize}

\subsection{Baselines}
RTI is evaluated against several state-of-the-art diffusion-based policies designed to enhance inference speed. Specifically, we compare it with image-based Consistency Policy ~\cite{prasad2024consistencypolicyacceleratedvisuomotor}, which utilizes consistency distillation for rapid sampling, and Streaming Diffusion Policy~\cite{heg2024streamingdiffusionpolicyfast}, which improves inference speed by generating a partially denoised action trajectory. 

Additionally, since Real-Time Scheme is built on a diffusion policy, we include the CNN-based Diffusion Policy~\cite{chi2024diffusionpolicyvisuomotorpolicy} as a baseline. For tasks involving grasping, we evaluate two RTI-DP setups: RTI-DP-scale and RTI-DP-clip. RTI-DP-scale applies scaling at the dataset level, while RTI-DP-clip is used for a pre-trained policy and tested on the same benchmarks as DP.

\subsection{Simulation Experiment}
We assess the performance of RTI-DP across seven tasks spanning three established benchmarks: Robomimic, Push-T, and Block-pushing. These benchmarks are widely used for both visuomotor and state-based policy learning and have been previously tested in studies such as Diffusion Policy~\cite{chi2024diffusionpolicyvisuomotorpolicy} and Consistency Policy~\cite{prasad2024consistencypolicyacceleratedvisuomotor}. 

\subsubsection{Robomimic~\cite{mandlekar2021matterslearningofflinehuman}}
Robomimic is a comprehensive benchmark for assessing robotic manipulation in imitation learning and offline reinforcement learning. It includes five tasks—can, lift, transport, tool hang, and square—encompassing both short and long horizons, single and dual-robot setups, high-precision actions, and scenarios involving single or multiple objects.

\subsubsection{Push-T~\cite{florence2021implicitbehavioralcloning}}
Push-T requires an agent to push a T-shaped block toward a fixed target using a circular end-effector. The task incorporates variability by randomizing the initial positions of both the block and the end-effector. Effective execution demands precise control over contact-rich interactions, utilizing point contacts to guide the block accurately. This task is available in two formats: one based on RGB image observations and another utilizing nine 2D keypoints derived from the ground-truth pose of the T block.

\subsubsection{Multimodal Block Pushing~\cite{shafiullah2022behaviortransformerscloningk} }
This task evaluates a policy’s ability to model multimodal action distributions by requiring the agent to push two blocks into two designated squares in any order.

\subsection{Implementation}
We assess inference time and success rates across three checkpoints, each initialized with different random seeds. All experiments are conducted using an NVIDIA A100-SXM4-40GB GPU, maintaining consistent hardware setup throughout. Policies are trained for a maximum of 48 hours, except for CP, which is trained for double the duration. To minimize data-related variability, all policies are trained using the dataset provided by~\cite{chi2024diffusionpolicyvisuomotorpolicy}. Additionally, DP and RTI-DP-clip are evaluated using the provided checkpoints from~\cite{chi2024diffusionpolicyvisuomotorpolicy}.

\subsection{Results}

The results, presented in Tables~\ref{table_state}, \ref{table_image}, and \ref{table_block}, demonstrate that our methods significantly accelerate inference while maintaining performance comparable to Diffusion Policy with full denoising steps. Compared to other state-of-the-art inference acceleration methods, our approach achieves substantial speedups while preserving a good performance, particularly in precision-critical tasks such as image-based Tool-Hang. Additionally, in tasks with only continuous action spaces, such as Push-T and BlockPush, our approach demonstrates a clear advantage over other acceleration methods, further highlighting its effectiveness and validating our method’s grounding in the continuity of physical systems.

Notably, in state-based tasks such as Push-T and BlockPush, RTI-DP even surpasses DP with full denoising steps, further demonstrating the advantage of our method and underscoring the importance of fully closed-loop control.

Most RTI-DP-scale results achieve effective denoising within just three steps. However, RTI-DP-clip exhibits slightly lower performance than RTI-DP-scale and requires additional time, as scaling directly on initial guesses degraded their quality.


\section{DISCUSSION}


RTI-DP offers a clear advantage over existing approaches by achieving substantial speedups without compromising performance, or requiring retraining, making it well-suited for real-time robotic applications. Beyond diffusion models, similar improvements could extend to flow-based models with a denoising component, further reducing computational overhead while maintaining expressive power. Future research could investigate how to integrate optimized initialization techniques with flow-based models, which might achieve even greater speedups than diffusion models, particularly given their inherent smoothness.


While fast inference offers many advantages, some dynamical systems with minimal environmental disturbances—such as static tasks like picking up a cup—do not necessarily benefit from speed improvements, as execution time is not a limiting factor. However, the efficiency of our method creates a larger time budget for prediction with less powerful GPUs. This, in turn, could potentially lower the hardware requirements for visuomotor policies and deployment in edge devices without needing additional training.

For a system with minimal noise, convergence is typically achieved within three steps. However, in environments with significant or unpredictable noise, additional tuning may be necessary. Dynamically adjusting the step sequence—such as selectively skipping steps—can sometimes reduce inference time while preserving performance. Although we provide a method for selecting the initial denoising step, the choice of the parameters still requires tuning. Also, in systems with large action changes, additional denoising steps may be required. Future work could explore strategies for selecting denoising steps that simplify the tuning process.

While most imitation learning methods depend on the quality of the teleoperated data, our method is even more reliant on it. Under the primary assumption of limited action changes over time, our method performs well when the demonstrated motion has low accelerations and relatively constant velocities. However, the method may struggle to reach the target actions within a limited number of inference steps when attempting to replicate sudden changes from the expert data. 

\section{CONCLUSIONS}

In this paper, we present the Real-Time Iteration Scheme for Diffusion Policy, a simple yet efficient approach to accelerating inference in diffusion-based policies through an informed initialization strategy. Using a reduced number of denoising steps, RTI-DP allows low-latency execution while maintaining good policy performance. Our method can be seamlessly integrated with other pre-trained models, enabling fast inference for large-scale pre-trained robotic models. We provide conditions for the contractivity of denoising dynamics and as potential guidelines for selecting the appropriate number of denoising steps. Simulation experiments demonstrate that RTI delivers both rapid inference and good performance simultaneously, without requiring retraining of existing policies.

Beyond diffusion models, our approach could also benefit flow-based models with denoising components, offering further opportunities to enhance inference speed. Additionally, while our method requires multiple steps for denoising, there remains a potential for one-step denoising without altering the policy structure.

\bibliographystyle{IEEEtran}
\bibliography{references}

\addtolength{\textheight}{-12cm}   








\end{document}